\def\year{2022}\relax
\documentclass[letterpaper]{article} 
\usepackage{aaai22}  
\usepackage{times}  
\usepackage{helvet}  
\usepackage{courier}  
\usepackage[hyphens]{url}  
\usepackage{graphicx} 
\usepackage{natbib}  
\usepackage{caption} 
\DeclareCaptionStyle{ruled}{labelfont=normalfont,labelsep=colon,strut=off} 
\newcommand{\methodname}{{\tt{FedCompetitors}}}

\newcommand{\rmnum}[1]{\romannumeral #1}
\newcommand{\Rmnum}[1]{\expandafter\@slowromancap\romannumeral #1@}
\makeatother
\usepackage{array}
\newcolumntype{C}[1]{>{\centering\arraybackslash}m{#1}}

\usepackage{xcolor}
\usepackage{graphicx}

\usepackage{float} 
\usepackage{subfigure}
\usepackage{amsmath}
\usepackage{breqn}

\usepackage{multirow}

\usepackage{amsthm}
\newtheorem{assump}{\textbf{Assumption}}
\newtheorem{proposition}{\textbf{Proposition}}
\newtheorem{lemma}{\textbf{Lemma}}

\usepackage[linesnumbered, ruled, vlined]{algorithm2e}
\usepackage{threeparttable}
\usepackage{indentfirst}

\usepackage[caption=false]{subfig}
\usepackage{amssymb}
\usepackage{graphicx,bm}
\usepackage{verbatim}

\usepackage{cite}
\title{\methodname: Harmonious Collaboration in Federated Learning with Competing Participants}
\author{
    Shanli~Tan\textsuperscript{\rm 1}\equalcontrib, Hao~Cheng\textsuperscript{\rm 2}\equalcontrib, Xiaohu~Wu\textsuperscript{\rm 1}\equalcontrib$^{\bigstar}$, Han~Yu\textsuperscript{\rm 3}\equalcontrib, Tiantian~He\textsuperscript{\rm 4}$^{\bigstar}$, Yew-Soon~Ong\textsuperscript{\rm 3,4},\\ Chongjun Wang\textsuperscript{\rm 2}, Xiaofeng~Tao\textsuperscript{\rm 1}
}
\affiliations{
    \textsuperscript{\rm 1}National Engineering Research Center of Mobile Network Technologies, Beijing University of Posts and Telecommunications, China\\
    \textsuperscript{\rm 2}State Key Laboratory for Novel Software Technology, Nanjing University, China\\
    \textsuperscript{\rm 3}School of Computer Science and Engineering, Nanyang Technological University, Singapore\\ 
    \textsuperscript{\rm 4}Agency for Science, Technology and Research, Singapore\\
    \textsuperscript{$^{\bigstar}$}Corresponding authors, emails: xiaohu.wu@bupt.edu.cn; he\_tiantian@cfar.a-star.edu.sg\\
}

\usepackage{bibentry}

\begin{document}

\SetKwInOut{Begin}{Begin}
\SetKwInOut{Input}{Input}
\SetKwInOut{Output}{Output}
\SetKw{Continue}{continue}
\SetKw{Break}{break}

\maketitle

\begin{abstract}
Federated learning (FL) provides a privacy-preserving approach for collaborative training of machine learning models. Given the potential data heterogeneity, it is crucial to select appropriate collaborators for each FL participant (FL-PT) based on data complementarity. Recent studies have addressed this challenge. Similarly, it is imperative to consider the inter-individual relationships among FL-PTs where some FL-PTs engage in competition. Although FL literature has acknowledged the significance of this scenario, practical methods for establishing FL ecosystems remain largely unexplored. In this paper, we extend a principle from the balance theory, namely ``the friend of my enemy is my enemy'', to ensure the absence of conflicting interests within an FL ecosystem. The extended principle and the resulting problem are formulated via graph theory and integer linear programming. A polynomial-time algorithm is proposed to determine the collaborators of each FL-PT. The solution guarantees high scalability, allowing even competing FL-PTs to smoothly join the ecosystem without conflict of interest. The proposed framework jointly considers competition and data heterogeneity. Extensive experiments on real-world and synthetic data demonstrate its efficacy compared to five alternative approaches, and its ability to establish efficient collaboration networks among FL-PTs.
\end{abstract}

\section{Introduction}
\label{sec.introduction}
Federated Learning (FL) represents a paradigm within distributed machine learning (ML) that facilitates the collaborative training of ML models by leveraging data from multiple parties while upholding privacy considerations \cite{Yang19a}. Each participant in FL (referred to as FL-PT) acts as a custodian of data and directly employs its dataset to locally train a model. In the well-established Federated Averaging (FedAvg) framework \cite{McMahan17a}, a central server (CS) periodically gathers model updates from individual FL-PTs, which are then aggregated to refine a global model. Similarly, each FL-PT regularly acquires the latest global model from the CS and further enhances it through local training. This iterative interplay between the CS and FL-PTs persists until the global model achieves convergence. FL has demonstrated significant promise across diverse domains, 
including healthcare, digital banking, ridesharing, recommender systems, and drug discovery \cite{sheller2020federated,long2020federated,Yang2020,10.1145/3534678.3539047,oldenhof2023industry,sun2023federated}.%

For example, consider a clinical research network of multiple hospitals \cite{Fleurence14a}. These hospitals possess the capacity to collaboratively construct ML models. In an optimal setting, the global model derived from FL should outperform models crafted by individual FL-PTs. However, a potential complication arises from the non-independent and non-identically distributed (Non-IID) nature of data across these FL-PTs \cite{Jin-Non-IID}. Each FL-PT undertakes local model training, which might lead it to a distinct local optima, diverging from the global optima. Consequently, the model performance of an FL-PT might experience degradation due to the FL process \cite{Wang19a}. The diversity in data characteristics among FL-PTs can be graphically portrayed using a directed benefit graph denoted as $\mathcal{G}_{b}$ \cite{Cui22a}. In this graphical representation, an edge from FL-PT $v_{i}$ to $v_{j}$ signifies that the data from $v_{i}$ can potentially enhance the learning outcomes of $v_{j}$ through the FL process.

Besides data heterogeneity, another important factor is the relationships among FL-PTs. For instance, in the context of hospitals located in different cities, they serve distinct populations. As depicted in Figure \ref{Fig-hospitals-relationship}, the hospital in city $C$ solely focuses on improving its own ML model, and its utility is independent of any FL-PT in other cities. Such two FL-PTs are considered ``independent", where the shared global model in FL functions as a public good, similar to a radio signal where each individual only values the received signal quality \cite{Tang21a}. In contrast, hospitals within the same city (e.g., city $B$) serve the same population, which can include both public and private hospitals. Then, competition arises where the utility of an FL-PT also depends on the model performance of its competitor \cite{brekke2011hospital}. Such FL-PTs are considered ``competitive". The inter-individual relationship between any two FL-PTs can be represented by an undirected graph $\mathcal{G}_{c}$. 

\begin{figure}[t]
\footnotesize
\centering
\includegraphics[width=0.7\columnwidth]{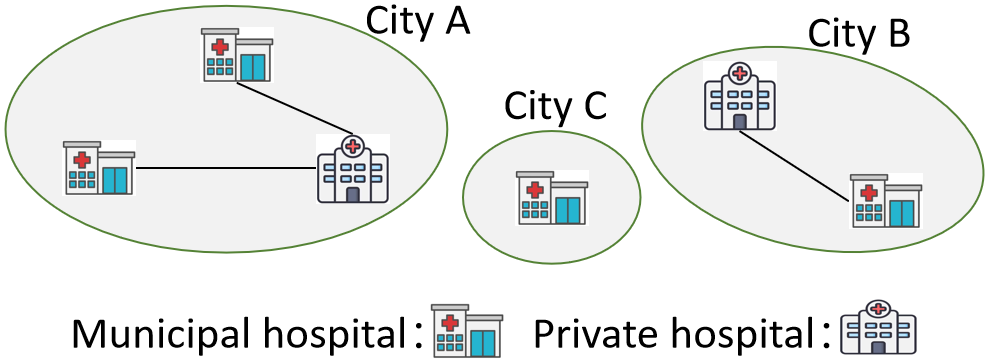}
\caption{Illustration of the Relationships among Hospitals: the black line denotes the competing relationship between two hospitals.}
\label{Fig-hospitals-relationship}
\end{figure}

In the presence of both data heterogeneity and competition, selecting suitable collaborators for each FL-PT is a crucial challenge. Recently, \citet{Cui22a} consider the data heterogeneity case (i.e., the edge set of $\mathcal{G}_{b}$ is non-empty and the edge set of $\mathcal{G}_{c}$ is empty) and leverages the concept of core-stable coalition from cooperative games to effectively address this. All FL-PTs are partitioned into disjoint groups/coalitions. Let $\pi(i)$ denote the coalition to which $v_{i}$ belongs where $\pi$ is called a coalition structure, and $v_{i}$'s utility depends on the FL-PTs in $\pi(i)$. For a core-stable coalition structure $\pi$, there is no other coalition $\mathcal{C}$ such that every FL-PT $v_i$ in $\mathcal{C}$ prefers $\mathcal{C}$ over $\pi(i)$ \cite{aziz_savani_moulin_2016}. Nevertheless, there is no existing work addressing the issue of competition among a part of FL-PTs when establishing collaborations in FL ecosystems. 

In this paper, we propose the \methodname{} approach to bridge this gap. It is general in the sense that (\rmnum{1}) the edge set of $\mathcal{G}_{c}$ is empty or non-empty except the complete graph case and (\rmnum{2}) the edge set of $\mathcal{G}_{b}$ is non-empty. The presence of competing FL-PTs has been recognized as an important aspect in the FL literature \cite{10.1561/2200000083,Zhan22a,ShiTNNLS23}. In balance theory, a principle, namely ``the friend of my enemy is my enemy", can avoid conflict of interest \cite{Jure10a,Dorwin56a}. We apply its extended version to establish collaboration among FL-PTs. Specifically, suppose $v_{i}$ and $v_{k}$ compete, and $v_{j}$ is the friend of $v_{i}$ (i.e., $v_{i}$ benefits from the data of $v_{j}$ in FL training). The FL-PT $v_{i}$, its friend $v_{j}$, and other FL-PTs who benefit $v_{i}$ and $v_{j}$ are in an alliance. Then, the CS regulates that $v_{k}$ will not make a contribution to any FL-PT in the alliance, which ensures that no FL-PTs directly or indirectly assist their competitors. 
If two FL-PTs can collaborate together, they are independent of each other. In a group of independent FL-PTs, an FL-PT can freely collaborate with other FL-PTs in the group, thereby maximizing the social welfare of the entire FL ecosystem.

The extended principle and the resulting problem above can be formulated via graph theory and integer linear programming. We further propose a polynomial-time algorithm that is to determine the collaborators of each FL-PT. Using the proposed solution, even competing FL-PTs can seamlessly join without conflict of interest and the FL ecosystem thus exhibits a high level of scalability and is trusted by FL-PTs with conflicting interests \cite{tariq2023trustworthy,yu2014reputation}. 
Extensive experiments on both synthetic and real-world datasets demonstrate the effectiveness of \methodname{} over the state of the art.


\section{Related Work}
\label{sec.related-work}
We focus on the context of cross-silo FL, where FL-PTs are typically companies or organizations and they both contribute their data and utilize the trained ML models. In the existing research, two scenarios have been extensively investigated: 
(\rmnum{1}) any two FL-PTs in the FL ecosystem are independent of each other and an FL-PT solely focuses on improving its own model performance, without considering the potential competition, and (\rmnum{2}) any two FL-PTs in the FL ecosystem compete against each other where $\mathcal{G}_{c}$ is a complete graph. 
In this paper, we mainly consider the scenario where there exists competition among a part of FL-PTs and an FL-PT will not collaborate with its competitors and other FL-PTs with potential conflict of interest.


{\em Firstly}, in the independent scenario, prior studies focus on alleviating the side effect of data heterogeneity. 
While applying Hedonic games that are a type of cooperative games \cite{aziz_savani_moulin_2016}, stable coalition structures are sought to establish collaboration among FL-PTs. 
\citet{Kleinberg21a} provide an analytical understanding of what partition of FL-PTs leads to a stable coalition structure for mean estimation and linear regression. \citet{chaudhury2022fairness} treat all FL-PTs as a grand coalition and optimizes a common model for all FL-PTs, which is considered core-stable if there is no other coalition $\mathcal{S}$ of FL-PTs that could significantly benefit by training a model with only their data. 
Another way that learns personalized models for FL-PTs works as follows \cite{tan2022towards}: (\rmnum{1}) use the CS to train a global model, and (\rmnum{2}) adapt the model to the local data of FL-PTs. Several approaches, such as meta-learning, and multi-task learning, have been employed for personalization \cite{fallah2020personalized,smith2017federated}. 
\citet{ding2022collaborative} study the case when the FL ecosystem expands to have numerous independent FL-PTs. A group of FL-PTs that has similar contributors is a group of collaboration partners. The authors propose to partition all FL-PTs into $K$ groups and adaptively learn a small number $K$ of models for $n$ FL-PTs, where $1\ll K\ll n$. 

{\em Secondly}, in the competition scenario, all FL-PTs are assumed to offer the same service in a given market. \citet{Wu22a} aim to achieve the objective of maintaining a negligible change in market share after FL-PTs join the FL ecosystem \cite{Farris10a,Wu23a}, and analyze the achievability of this objective. Afterwards, two other works study the profitablity of FL-PTs in the given market after FL-PTs join the FL ecosystem, but are taken under different assumptions on the source of extra profit brought by FL. Specifically, \citet{tsoy2023strategic} use the following assumption: (\rmnum{1}) each consumer has a fixed budget that is allocated to multiple services from different markets, and (\rmnum{2}) if an FL-PT has a higher model quality, its service quality is higher and the consumer will allocate more of its budget to consume the service. 
\citet{huang2023duopoly} consider duopoly business competition between two FL-PTs and assume that, if the model-related service can be improved by FL, customers will have willingness to pay more and FL-PTs thus have opportunities to increase their profits. 

\section{Model and Assumptions}

We use graph theory to describe our model of interest and mathematically formulate the extended principle. 
Specifically, let us consider a set of $n$ FL-PTs denoted by $\mathcal{V}=\{v_{1},$ $v_{2}, \cdots, v_{n}\}$. Each FL-PT $v_{i}$ possesses a local dataset $\mathcal{D}_{i}$. The FL-PTs contemplate joining a collaborative FL network, facilitated by the CS. However, challenges such as data heterogeneity and competition arise among the FL-PTs. To characterize the various relationships among the FL-PTs, three graphs are employed.



\textbf{Competing graph $\mathcal{G}_{c}$.} 
An undirected graph $\mathcal{G}_{c}=(\mathcal{V}, E_{c})$ is used to represent the competing relations between any two FL-PTs, where $\mathcal{V}$ is the set of nodes/FL-PTs and $E_{c}$ is the set of edges. An edge $(v_{i}, v_{j})\in E_{c}$ signifies a competitive relationship between FL-PTs $v_{i}$ and $v_{j}$. The adjacency matrix of $\mathcal{G}_{c}$ is denoted as $S_{n\times n}$: its main diagonal elements are set to zero, i.e., $s_{i,i}=0$; when $i\neq j$, $s_{i,j}=1$ if $v_i$ competes with $v_{j}$, and $s_{i,j}=0$ if $v_i$ is independent of $v_{j}$. Each FL-PT $v_{i}$ will report its competitors to CS, as it hopes that CS will correctly utilize this information to prevent its competitors from benefiting from its data. Thus, CS has the knowledge of $\mathcal{G}_{c}$.


\textbf{Benefit graph $\mathcal{G}_{b}$.} 
A benefit graph is employed to depict the impact of sample distribution discrepancies among the $n$ FL-PTs. For any two FL-PTs $v_i$ and $v_j$, if $w_{j,i}=0$, it indicates that $v_i$ cannot benefit from the data of $v_j$. Conversely, if $w_{j,i}>0$, it implies that $v_i$ can benefit from $v_j$'s data, with larger values of $w_{j,i}$ signifying greater benefit to $v_i$. These values $w_{j,i}$ define a directed graph denoted as $\mathcal{G}_{b}=(\mathcal{V}, E_{b})$, referred to as the benefit graph: $(v_{j}, v_{i})\in E_{b}$ if and only if $i\neq j$ and $w_{j,i}>0$. 
The adjacency matrix of $\mathcal{G}_{b}$ is denoted as $W_{n\times n}$, where the $i$-th column comprises the weights $w_{1,i}, w_{2,i}, \cdots, w_{n,i}$, representing the importance of the $n$ FL-PTs to $v_{i}$. The level of potential (LoP) of an FL-PT $v_{i}$ contributing to the other FL-PTs $\mathcal{V}-\{v_{i}\}$ is defined as
\begin{align}\label{equa-contribution-to-others}
w_{i} = \sum\nolimits_{j\neq i}{w_{i,j}},
\end{align}
which measures the importance of $v_{i}$ to the FL ecosystem.  
The graph $\mathcal{G}_{b}$ can be obtained by the hypernetwork technique in \cite{Cui22a,navon2021learning}.



%


\textbf{Data usage graph $\mathcal{G}_{u}$.} 
Although $v_{i}$ may benefit from $v_{j}$'s data ($w_{j,i}>0$), CS has the authority to determine whether $v_{i}$ can actually utilize $v_{j}$'s local model update information (i.e., indirectly use $v_{j}$'s data) in the FL training process or not. Let $X=(x_{j,i})$ be a $n\times n$ matrix where 
\begin{align}\label{equa-decision-variables}
x_{j,i}\in \{0,1\}
\end{align}
is a decision variable: for two different FL-PTs $v_{i}$ and $v_{j}$, $x_{j,i}$ is set to one if $v_j$ will contribute to $v_i$ (i.e., $v_i$ will utilize $v_{j}$'s local model update information) in the FL training process and $x_{j,i}$ is set to zero otherwise. $X$ defines a directed graph $\mathcal{G}_{u}=(\mathcal{V}, E_{u})$, called the data usage graph: $(v_{j}, v_{i})\in E_{u}$ if and only if $j\neq i$ and $x_{j,i}=1$; then, $v_j$ is said to be a collaborator or friend of $v_i$. Consider any pair of FL-PTs $v_i$ and $v_j$. If $v_j$'s data cannot benefit $v_i$ ($w_{j,i}=0$), we set $x_{j,i}=0$. 
Only when $v_j$'s data can benefit $v_i$, there is a possibility that $x_{j,i}=1$. Consequently, $E_{u}$ is a subset of $E_{b}$, leading directly to the following conclusion.

\begin{lemma}\label{lemma-unreachability}
For any two nodes $v_{j}$ and $v_{i}$, if there is no path from $v_{j}$ to $v_{i}$ in the benefit graph $\mathcal{G}_{b}$, then this also holds in the data usage graph $\mathcal{G}_{u}$.   
\end{lemma}

\subsection{Principle for avoiding conflict of interest}

Below, we extend the principle that ``the friend of my enemy is my enemy”. 

\begin{assump}\label{assump-data-usage-constraint-1}
For any two competing FL-PTs $v_{i}$ and $v_{j}$ (i.e., $(v_{i},v_{j})\in E_{c}$), $v_{j}$ is unreachable to $v_{i}$ in the data usage graph $\mathcal{G}_{u}$. 
\end{assump}

\begin{figure}[t!]
\centering
\includegraphics[width=0.83\columnwidth]{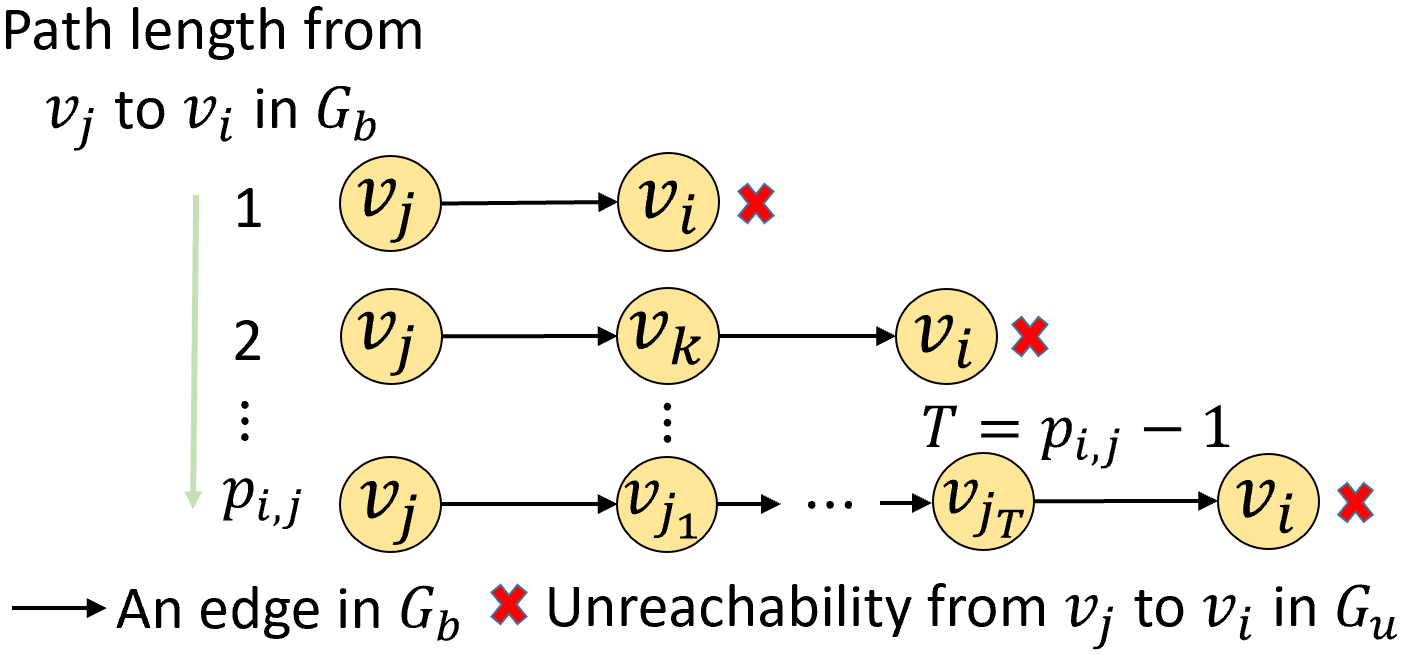}
\caption{Illustration of Assumption \ref{assump-data-usage-constraint-1}: $v_{j}$ is reachable to $v_{i}$ in $\mathcal{G}_{b}$, while $v_{i}$ and $v_{j}$ compete against each other.}
\label{Fig.priciple}
\end{figure}


Assumption \ref{assump-data-usage-constraint-1} is implemented while establishing the collaboration relationships among FL-PTs. Suppose there is a path from $v_{j}$ to $v_{i}$ in the benefit graph $\mathcal{G}_{b}$ whose length is $p_{i,j}$. We use Figure \ref{Fig.priciple} to explain the implication of Assumption \ref{assump-data-usage-constraint-1}. If $p_{i,j}=1$, it posits that one FL-PT refuses to contribute to its competitor. If $p_{i,j}=2$, we use $v_{k}$ to denote the intermediate node between $v_{j}$ and $v_{i}$. 
If $v_{i}$ benefits from $v_{k}$, $v_{k}$ is $v_{i}$'s friend; $v_{j}$ is not willing to see the enhancement of $v_{i}$'s model and will threaten not to contribute to $v_{k}$. 
Assumption \ref{assump-data-usage-constraint-1} posits that, if $(v_{k}, v_{i})\in E_{u}$, then $(v_{j}, v_{k})\notin E_{u}$, i.e., $v_{j}$ doesn't help the friend $v_{k}$ of its enemy $v_{i}$. Generally, for any $p_{i,j}$, 
the path from $v_{j}$ to $v_{i}$ in $\mathcal{G}_{b}$ is denoted as 
\begin{align}\label{def-P-j-i}
P_{j}^{i} = (v_{j_{0}}, v_{j_{1}}, \cdots, v_{j_{p_{i,j}}}), 
\end{align}
where $j_{0}=j$ and $j_{p_{i,j}}=i$. 
If any, let $t$ be the minimum integer in $[1, p_{i,j}-1]$ such that $(v_{j_{l}}, v_{j_{l+1}})\in E_{u}$ for every $l\in [t, p_{i,j}-1]$ where $v_{j_{l}}$ helps $v_{j_{l+1}}$. Then, FL-PTs $v_{j_{t}},$ $v_{j_{t+1}},$ $\cdots,$ $v_{j_{p_{i,j}}}$ are said to be in an alliance, and $v_{j}$ will not help any member in this alliance. 
Assumption \ref{assump-data-usage-constraint-1} follows a common logic in reality that nobody wants to see others help its enemy and its enemy's friends. By applying Assumption \ref{assump-data-usage-constraint-1}, it is strictly guaranteed that each FL-PT will not make a contribution to its competitors directly or indirectly.

\begin{figure*}[t]
\begin{minipage}[b]{0.496\textwidth}
\begin{center}
\subfigure[$v_{i}$ is reachable to the red node in the oval, which is also the competitor of the blue nodes.]{%
  \includegraphics[width=0.73\textwidth]{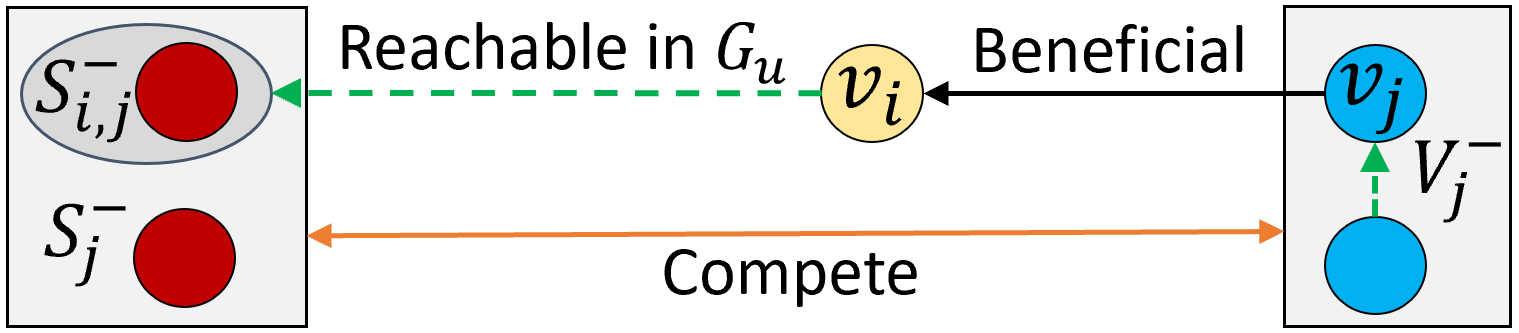}%
}%
\end{center}
\end{minipage}
\begin{minipage}[b]{0.496\textwidth}
\begin{center}
\subfigure[$v_{j}$ is reachable from the red node in the oval, which is also the competitor of the golden nodes.]{%
  \includegraphics[width=0.73\textwidth]{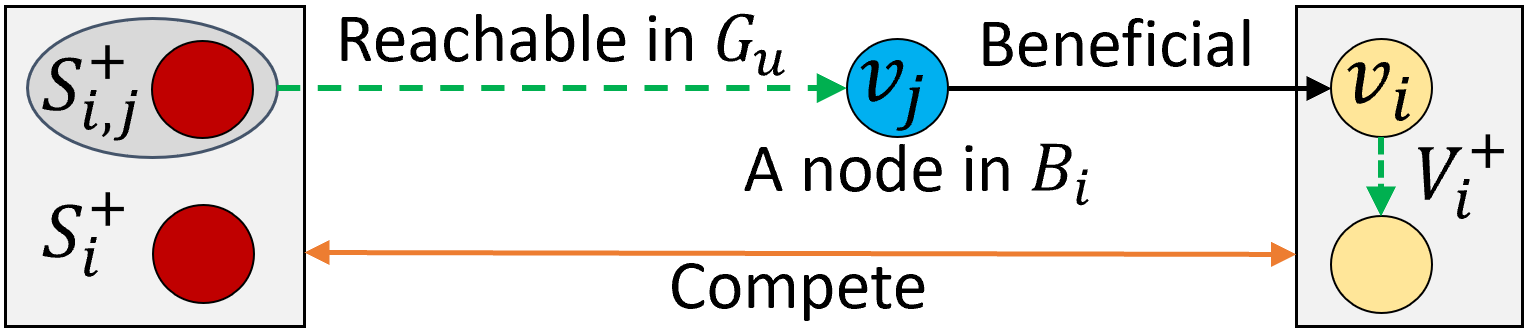}%
}
\end{center}
\end{minipage}
\caption{Effect on Assumption \ref{assump-data-usage-constraint-1} after adding an edge $(v_{j}, v_{i})$ in the data usage graph $\mathcal{G}_{u}$.}
\label{Fig-factors-collaborators-selection}
\end{figure*}

For any competing FL-PTs $v_{i}$ and $v_{j}$, let $\mathcal{P}_{j,i}$ denote the set of all reachable paths from $v_{j}$ to $v_{i}$ in the graph $\mathcal{G}_{b}$. Assumption \ref{assump-data-usage-constraint-1} can be characterized by $\mathcal{G}_{c}$, $\mathcal{G}_{b}$, and $\mathcal{G}_{u}$. 

\begin{proposition}\label{proposi-equivalent}
Assumption \ref{assump-data-usage-constraint-1} holds if and only if the following condition is satisfied:
\begin{align}
x_{j,j_{1}}  + x_{j_{1},j_{2}} + & \cdots +  x_{j_{p_{i,j}},i} \leqslant p_{i,j}-1,\label{equa-constraint-assump-2}\\ & \forall (v_{i}, v_{j})\in E_{c},\, \forall P_{j}^{i}\in \mathcal{P}_{j,i}. \nonumber
\end{align}
\end{proposition}
\begin{proof}
Firstly, we prove the reverse direction. By Lemma \ref{lemma-unreachability}, to satisfy Assumption \ref{assump-data-usage-constraint-1}, we only need to focus on such $v_{j}$ and $v_{i}$ that are reachable in $\mathcal{G}_{b}$. 
$P_{j}^{i}$ is defined in Eq. (\ref{def-P-j-i}). If Eq. (\ref{equa-constraint-assump-2}) holds, then, for any $P_{j}^{i}\in \mathcal{P}_{j,i}$ there exist two adjacent nodes $v_{j_{l}}$ and $v_{j_{l+1}}$ in $P_{j}^{i}$, where $l\in [0, p_{i,j}-1]$, such that $x_{j_{l},j_{l+1}}$$=0$ and $(v_{j_{l}}, v_{j_{l+1}})\notin E_{u}$. Thus, there are no reachable paths from $v_{j}$ to $v_{i}$ in $\mathcal{G}_{u}$ and Assumption \ref{assump-data-usage-constraint-1} is satisfied. Secondly, we prove the forward direction by contradiction. The length of $P_{j}^{i}$ is $p_{i,j}$. If Eq. (\ref{equa-constraint-assump-2}) doesn't hold, then, for any $l\in [0, p_{i,j}-1]$, $x_{j_{l},j_{l+1}}=1$ and there exists an edge from $v_{j_{l}}$ to $v_{j_{l+1}}$ in the graph $\mathcal{G}_{u}$, which contradicts Assumption \ref{assump-data-usage-constraint-1} where $v_{j}$ is not reachable to $v_{i}$ in $\mathcal{G}_{u}$.  
\end{proof}

In this paper, we aim to propose a framework that can construct an FL ecosystem without conflict of interest. Mathematically, our problem is to determine the matrix $X_{n\times n}$ of decision variables that satisfy Eq. (\ref{equa-decision-variables}) and (\ref{equa-constraint-assump-2}), which determines the collaborators of FL-PTs. Eq. (\ref{equa-constraint-assump-2}) is equivalent to Assumption \ref{assump-data-usage-constraint-1} by Proposition \ref{proposi-equivalent}. The absence of conflicting interests among FL-PTs is guaranteed by Eq. (\ref{equa-constraint-assump-2}).

\section{Polynomial-Time Algorithm}
\label{sec.algo-collaboration-formation}

We propose a polynomial-time algorithm to determine the matrix $X_{n\times n}$ of decision variables subject to Eq. (\ref{equa-decision-variables}) and (\ref{equa-constraint-assump-2}). We begin by describing the algorithm's initial states. 
The LoP $w_{i}$ in Eq. (\ref{equa-contribution-to-others}) measures the importance of $v_{i}$ to the FL ecosystem. We sort the LoPs of all FL-PTs in non-increasing order, and without loss of generality, we assume:
\begin{align}\label{equa-benefit-sequence}
w_1 \geqslant w_2  \geqslant \cdots \geqslant w_n. 
\end{align}
The initial values of $X_{n\times n}$ are set as follows: 
\begin{align}\label{equa-X-initial-values}
x_{j,i} =  1  \text{ if } i=j, \text{ and }  x_{j,i} =  0  \text{ if } i\neq j.
\end{align}
This defines the initial $\mathcal{G}_{u}$, which will be updated as the algorithm runs. We also define a connectivity matrix $C_{n\times n}$ of $\mathcal{G}_{u}$: when $i\neq j$, $c_{j,i}=1$ if there is a path from $v_{j}$ to $v_{i}$ and $c_{j,i}=0$ otherwise; $c_{i,i}$ is always set to one trivially. Initially, $C_{n\times n}$ is set as an identity matrix, i.e., a diagonal matrix whose main diagonal elements are all one.

\begin{algorithm}[t]
\caption{Collaborator Selection}\label{Greedy}
\KwData{$S_{n\times n}$, and $W_{n\times n}$}
\KwResult{$X_{n\times n}$}

Initialize $X_{n\times n}$ by Eq. (\ref{equa-X-initial-values}) and $C_{n\times n}$ to be an identity matrix\; 

Generate the sorted sequence (i.e., Eq. (\ref{equa-benefit-sequence}))\;

\For{$v_{i}$ in the sorted sequence}{

    Solve the ILP problem (\ref{equa-obj}) by Algorithm \ref{algo-ILP-solver}\;

}
\end{algorithm}

The proposed algorithm is presented as Algorithm \ref{Greedy}. The $n$ FL-PTs are considered sequentially from $v_1$ to $v_n$ (line 3). 
At the step for $v_{i}$ (line 4), the decision variables to be determined are $\{x_{j,i}\}_{j\neq i}$ and we maximize the benefit of $v_i$: 
\begin{align}\label{equa-obj}
\text{maximize}\enskip \sum\nolimits_{j\neq i}{w_{j,i}\cdot x_{j,i}} 
\end{align}
subject to Eq. (\ref{equa-decision-variables}) and (\ref{equa-constraint-assump-2}). Afterwards, $X_{n\times n}$ is updated and the collaborators of $v_{i}$ are determined. Next, we solve the integer linear programming (ILP) problem (\ref{equa-obj}). Let $\mathcal{B}_{i}$ denote all FL-PTs that can benefit $v_{i}$ but are independent of $v_{i}$, which can be defined by the adjacency matrix $W_{n\times n}$ of $\mathcal{G}_{b}$ and the adjacency matrix $S_{n\times n}$ of $\mathcal{G}_{c}$:
\begin{align}\label{equa-B-i}
\mathcal{B}_{i} = \left\{v_{j}\in \mathcal{V}\, |\, j\neq i, w_{j,i}>0, s_{j,i}=0\right\}. 
\end{align} 
$\mathcal{B}_{i}$ includes all possible collaborators of $v_{i}$.

For any $v_{j}\in\mathcal{B}_{i}$, let $\mathcal{V}_{j}^{-}$ denote a set consisting of all nodes that are reachable to $v_{j}$ in $\mathcal{G}_{u}$, as well as $v_{j}$ itself, which can be defined by the connectivity matrix $C_{n\times n}$:
\begin{align}\label{equa-V-j}
\mathcal{V}_{j}^{-} = \left\{v_{k}\in \mathcal{V}\, |\,  c_{k,j}=1\right\}.
\end{align} 
Let $\mathcal{S}_{j}^{-}$ denote all competitors of the nodes in $\mathcal{V}_{j}^{-}$, and $\mathcal{S}_{i,j}^{-}$ denote the nodes of $\mathcal{S}_{j}^{-}$ that are reachable from $v_{i}$ in $\mathcal{G}_{u}$:
\begin{align}
 \mathcal{S}_{j}^{-} & = \left\{v_{k}\in \mathcal{V}\, |\, \exists v_{p}\in \mathcal{V}_{j}^{-}:  s_{k,p}=1\right\},\label{equa-S-j} \\
 \mathcal{S}_{i,j}^{-} & = \left\{v_{k}\in \mathcal{S}_{j}^{-}\, |\,  c_{i,k}=1\right\}  \subseteq \mathcal{S}_{j}^{-}.\label{equa-S-i-j-1}
\end{align} 
As illustrated in Figure \ref{Fig-factors-collaborators-selection}(a), if $\mathcal{S}_{i,j}^{-}\neq \emptyset$, we have $x_{j,i}=0$; otherwise, some nodes in $\mathcal{V}_{j}^{-}$ will be reachable to its competitor (e.g., the node in the oval) in $\mathcal{G}_{u}$, which violates Eq. (\ref{equa-constraint-assump-2}). Let $\mathcal{V}_{i}^{+}$ denote a set consisting of all nodes that are reachable from $v_{i}$ in $\mathcal{G}_{u}$, as well as $v_{i}$ itself:
\begin{align}\label{equa-V-i}
\mathcal{V}_{i}^{+} = \{v_{k}\in \mathcal{V}\, |\,  c_{i,k}=1\}. 
\end{align} 
Let $\mathcal{S}_{i}^{+}$ denote all competitors of the nodes in $\mathcal{V}_{i}^{+}$, and $\mathcal{S}_{i,j}^{+}$ denote the nodes of $\mathcal{S}_{i}^{+}$ that are reachable to $v_{j}$ in $\mathcal{G}_{u}$: 
\begin{align}
 \mathcal{S}_{i}^{+} & = \{v_{k}\in \mathcal{V}\, |\, \exists v_{p}\in \mathcal{V}_{i}^{+}:  s_{p,k}=1\},\label{equa-S-i} \\
 \mathcal{S}_{i,j}^{+} & = \{v_{k}\in \mathcal{S}_{i}^{+}\, |\,  c_{k,j}=1\}\subseteq \mathcal{S}_{i}^{+}.\label{equa-S-i-j-2}
\end{align} 
Here, by Eq. (\ref{equa-V-j}), (\ref{equa-S-i-j-1}), (\ref{equa-V-i}), and (\ref{equa-S-i-j-2}), we have 
\begin{align}\label{equal-reachable-competitors}
 \mathcal{S}_{i,j}^{-} = \mathcal{V}_{i}^{+}\cap \mathcal{S}_{j}^{-} \subseteq \mathcal{V}_{i}^{+} \text{ and } \mathcal{S}_{i,j}^{+} = \mathcal{V}_{j}^{-}\cap \mathcal{S}_{i}^{+} \subseteq \mathcal{V}_{j}^{-}.
\end{align}
As illustrated in Figure \ref{Fig-factors-collaborators-selection}(b), if $\mathcal{S}_{i,j}^{+}\neq \emptyset$, then $x_{j,i}=0$; 
otherwise, some nodes in $\mathcal{V}_{i}^{+}$ will be reachable from its competitor (e.g., the node in the oval) in $\mathcal{G}_{u}$, violating Eq. (\ref{equa-constraint-assump-2}).

Based on the above understanding, we propose Algorithm \ref{algo-ILP-solver} to solve the ILP problem (\ref{equa-obj}). 
For a node $v_{j}\in\mathcal{B}_{i}$, $w_{j,i}$ represents the importance of $v_{j}$ to $v_{i}$. We sort the nodes of $\mathcal{B}_{i}$ in the non-increasing order of their values $w_{j,i}$ (line 1). The nodes of $\mathcal{B}_{i}$ are considered sequentially in this order (line 2). For each node $v_{j}\in\mathcal{B}_{i}$, if $\mathcal{S}_{i,j}^{+}=\emptyset$ and $\mathcal{S}_{i,j}^{-}=\emptyset$, the algorithm sets $v_{j}$ as the collaborator of $v_{i}$ (i.e., $x_{j,i}=1$), with the connectivity from $v_{j}$ to $v_{i}$ is updated (lines 3-4). Finally, we consider the effect of setting $x_{j,i}=1$ on the connectivity between any two nodes $v_{p}$ and $v_{q}$ in the graph $\mathcal{G}_{u}$, except $(v_{j}, v_{i})$ (line 5). In the graph $\mathcal{G}_{u}$, if we have before executing line 4 that $v_{p}$ is not reachable to $v_{q}$, $v_{p}$ is reachable to $v_{j}$, and $v_{i}$ is reachable to $v_{q}$, then $v_{p}$ becomes reachable to $v_{q}$ (lines 6-7).


\begin{lemma}\label{lemma-time-complexity}
Given $W_{n\times n}$, $S_{n\times n}$ and $C_{n\times n}$, the time complexity of finding $\mathcal{B}_{i}$ is $\mathcal{O}(n)$ while  the time complexity of finding $\mathcal{S}_{i,j}^{-}$ or $\mathcal{S}_{i,j}^{+}$ is $\mathcal{O}(n^{2})$.
\end{lemma}
\begin{proof}
By Eq. (\ref{equa-B-i}), the time complexity of finding $\mathcal{B}_{i}$ is $\mathcal{O}(n)$ where $|\mathcal{B}_{i}|\leqslant n$. By Eq. (\ref{equa-V-j}), the time complexity of finding $\mathcal{V}_{j}^{-}$ is $\mathcal{O}(n)$ where $|\mathcal{V}_{j}^{-}|\leqslant n$. By Eq. (\ref{equa-S-j}), $\mathcal{S}_{j}^{-}$ can be found by (\rmnum{1}) checking every $v_{k}\in\mathcal{V}$ and (\rmnum{2}) judging whether there exists a node $v_{p}\in\mathcal{V}_{j}^{-}$ such that $s_{k,p}=1$; the resulting time complexity is $\mathcal{O}(n^{2})$; here, $|\mathcal{S}_{j}^{-}|\leqslant n$. Given $\mathcal{S}_{j}^{-}$, by Eq. (\ref{equa-S-i-j-1}), the time complexity of finding $\mathcal{S}_{i,j}^{-}$ is $\mathcal{O}(n)$. Finally, the time complexity of finding $\mathcal{S}_{i,j}^{-}$ is $\mathcal{O}(n^{2})$. Similarly to $\mathcal{S}_{i,j}^{-}$, the time complexity of finding $\mathcal{S}_{i,j}^{+}$ is also $\mathcal{O}(n^{2})$.   
\end{proof}

\begin{algorithm}[t]
\caption{ILP Solver}\label{algo-ILP-solver}
\KwData{$W_{n\times n}$, $S_{n\times n}$, and $C_{n\times n}$}
\KwResult{the updated $X_{n\times n}$, and $C_{n\times n}$}



Sort the nodes of $\mathcal{B}_{i}$ in non-increasing order of their values $w_{j,i}$, generating a sorted sequence\;

\For{$v_{j}$ in the sorted sequence}{
   \If{$\mathcal{S}_{i,j}^{+}=\emptyset$ $\wedge$ $\mathcal{S}_{i,j}^{-}=\emptyset$}{
       $x_{j,i}\leftarrow 1$, $c_{j,i}\leftarrow 1$\;
       \For{any two integers $p\in [1, n]$ and $q\in [1, n]$ with $p\neq q$ and $(p,q)\neq (j, i)$}{
               \If{$c_{p,q}=0$ $\wedge$ $c_{p,j}=1 \wedge c_{i,q}=1$}{
                   $c_{p,q}\leftarrow 1$\;
           }
       }
   }
}
\end{algorithm}


\begin{proposition}\label{proposi-ILP-solver}
Suppose $X_{n\times n}$ satisfies Eq. (\ref{equa-decision-variables}) and (\ref{equa-constraint-assump-2}) before $v_{i}$ is considered. Algorithm \ref{algo-ILP-solver} gives a feasible solution to the ILP problem \eqref{equa-obj} with a time complexity $\mathcal{O}(n^{3})$ when $v_{i}$ is considered.
\end{proposition}
\begin{proof}
By Proposition \ref{proposi-equivalent}, Eq. (\ref{equa-constraint-assump-2}) is equivalent to Assumption \ref{assump-data-usage-constraint-1}. Firstly, we prove by contradiction that Algorithm \ref{algo-ILP-solver} gives a feasible solution. 
Before $v_{i}$ is considered, no two competitors in $\mathcal{V}$ are reachable in $\mathcal{G}_{u}$ by Assumption \ref{assump-data-usage-constraint-1}. Setting $x_{j,i}=1$ is equivalent to adding an edge $(v_{j}, v_{i})$ in $\mathcal{G}_{u}$. By the definition of $\mathcal{V}_{j}^{-}$ and $\mathcal{V}_{i}^{+}$, the addition of $(v_{j}, v_{i})$ can only affect the reachability from the nodes of $\mathcal{V}_{j}^{-}$ to the nodes of $\mathcal{V}_{i}^{+}$ in $\mathcal{G}_{u}$. Suppose there exists a node $v_{j}\in \mathcal{B}_{i}$ satisfying $\mathcal{S}_{i,j}^{+}=\emptyset$ and $\mathcal{S}_{i,j}^{-}=\emptyset$, such that, Assumption \ref{assump-data-usage-constraint-1} is violated after setting $x_{j,i}=1$. 
Thus, the addition of $(v_{j}, v_{i})$ leads to that some node of $\mathcal{V}_{j}^{-}$ is reachable to and competes with some node of $\mathcal{V}_{i}^{+}$ in $\mathcal{G}_{u}$. Then, there exists a node $v_{k}$ such that either $v_{k}\in \mathcal{V}_{j}^{-}$ and $v_{k}$ is a competitor of some node in $\mathcal{V}_{i}^{+}$ (i.e., $v_{k}\in \mathcal{S}_{i,j}^{+}$ by Eq. (\ref{equa-S-i}) and (\ref{equal-reachable-competitors})), or $v_{k}\in \mathcal{V}_{j}^{+}$ and $v_{k}$ is a competitor of the nodes of $\mathcal{V}_{i}^{-}$ (i.e., $v_{k}\in \mathcal{S}_{i,j}^{-}$ by Eq. (\ref{equa-S-j}) and (\ref{equal-reachable-competitors})). 
$\mathcal{S}_{i,j}^{-}$ and $\mathcal{S}_{i,j}^{+}$ are non-empty, which contradicts the condition in line 3 that leads to $x_{j,i}=1$.

Secondly, we show the complexity of Algorithm \ref{algo-ILP-solver}. Given $\mathcal{B}_{i}$, the time complexity of sorting the nodes of $\mathcal{B}_{i}$ is $\mathcal{O}(n\log{n})$, e.g., using the mergesort algorithm. Thus, by Lemma \ref{lemma-time-complexity}, the time complexity in line 1 is $\mathcal{O}(n\log{n})$. For the for-loop in line 2, its time complexity is $\mathcal{O}(n)$ where $|\mathcal{B}_{i}|\leqslant n$; by Lemma \ref{lemma-time-complexity}, the time complexity in line 3 is $\mathcal{O}(n^{2})$. For the for-loop in line 5, the time complexity is $\mathcal{O}(n^{2})$. The total time complexity in lines 2--7 is $\mathcal{O}(n^{3})$. Finally, Algorithm \ref{algo-ILP-solver} has a time complexity $\mathcal{O}(n^{3})$. 
\end{proof}

We show the correctness of Algorithm \ref{Greedy}. At the beginning of Algorithm \ref{Greedy}, $X_{n\times n}$ satisfies Eq. (\ref{equa-decision-variables}) and (\ref{equa-constraint-assump-2}) by Eq. (\ref{equa-X-initial-values}). After each step for $v_{i}$ in line 4, $X_{n\times n}$ still satisfies these constraints by Proposition \ref{proposi-ILP-solver}. When Algorithm \ref{Greedy} ends, the final collaborating relationship among all FL-PTs is determined by $X_{n\times n}$. By Eq. \eqref{equa-contribution-to-others}, the time complexity of computing $w_{i}$ for each FL-PT $v_{i}$ is $\mathcal{O}(n)$; thus, the time complexity of computing $w_{1}, w_{2}, \cdots, w_{n}$ is $\mathcal{O}(n^{2})$. The time complexity of sorting $w_{1}, w_{2}, \cdots, w_{n}$ is $\mathcal{O}(n\log{n})$. Thus, the time complexity in line 2 of Algorithm \ref{Greedy} is $\mathcal{O}(n^{2})$. By Proposition \ref{proposi-ILP-solver}, the time complexity in lines 3-4 is $\mathcal{O}(n^{4})$. Thus, the time complexity of Algorithm \ref{Greedy} is $\mathcal{O}(n^{4})$.


\begin{table*}[t]
	\centering
		\caption{Experiments with synthetic data under fixed competing graphs}
	\begin{threeparttable}[b]
 \small
		\begin{tabular}{|C{2.66cm}| C{1.38cm} | C{1.38cm} | C{1.38cm} | C{1.38cm} | C{1.38cm} |C{1.38cm}| C{1.38cm} | C{1.38cm} |}
			\hline
\multicolumn{9}{|c|}{Weakly Non-IID setting (MSE)} \\ \cline{1-9} 
     &  $v_{1}$   &   $v_{2}$   &   $v_{3}$   &  $v_{4}$   &   $v_{5}$    &   $v_{6}$    &   $v_{7}$   &  $v_{8}$    \\ \hline

        Local     &   0.23$\pm$0.08   &   0.23$\pm$0.09    &  0.87$\pm$0.41     &   0.82$\pm$0.26   &   0.23$\pm$0.10  &  0.23$\pm$0.07   &   0.82$\pm$0.24    &  0.78$\pm$0.30      \\ 

        FedAvg    &   0.20$\pm$0.06   &   0.20$\pm$0.06    &  0.20$\pm$0.10     &   0.19$\pm$0.07   &   0.19$\pm$0.06  &  0.19$\pm$0.06   &   0.19$\pm$0.08    &  0.19$\pm$0.10      \\ 

        FedProx     &   0.16$\pm$0.06   &   0.17$\pm$0.07    &  0.15$\pm$0.09     &   0.17$\pm$0.08   &   0.17$\pm$0.06  &  0.17$\pm$0.06   &   0.16$\pm$0.09    &  0.18$\pm$0.07      \\ 

        SCAFFOLD    &   0.17$\pm$0.07   &   0.17$\pm$0.07    &  0.16$\pm$0.09     &   0.16$\pm$0.07   &   0.18$\pm$0.06  &  0.18$\pm$0.07   &   0.18$\pm$0.08    &  0.18$\pm$0.08      \\ 
      
        CE     &    \textbf{0.14$\pm$0.10}   &   0.14$\pm$0.11    &  1.14$\pm$0.67     &   1.20$\pm$0.88   &    \textbf{0.15$\pm$0.08}  &  0.16$\pm$0.09   &   1.23$\pm$0.37    &  1.22$\pm$0.81      \\ \hline
        
        \methodname{}  &   0.14$\pm$0.12   &   \textbf{0.14$\pm$0.07}    &  \textbf{0.13$\pm$0.06}   &    \textbf{0.15$\pm$0.06}    &   \textbf{0.15$\pm$0.08}   &    \textbf{0.14$\pm$0.06}    &   \textbf{0.14$\pm$0.07}    &   \textbf{0.14$\pm$0.07}       \\ \hline\hline

\multicolumn{9}{|c|}{Strongly Non-IID Setting (MSE)} \\ \cline{1-9} 
     &  $v_{1}$   &   $v_{2}$   &   $v_{3}$   &  $v_{4}$   &   $v_{5}$    &   $v_{6}$    &   $v_{7}$   &  $v_{8}$    \\ \hline

        Local     &   0.23$\pm$0.08   &   0.23$\pm$0.08    &  0.22$\pm$0.07     &   0.23$\pm$0.08   &   0.23$\pm$0.06  &  0.22$\pm$0.06   &   0.22$\pm$0.08    &  0.23$\pm$0.07      \\ 

        FedAvg    &   24.47$\pm$4.98   &   24.85$\pm$4.82    &  24.85$\pm$5.03     &   24.73$\pm$5.67   &   24.15$\pm$3.00  &  24.47$\pm$2.78   &   24.17$\pm$4.40    &  24.97$\pm$3.81      \\ 

        FedProx     &   17.80$\pm$7.54   &   17.82$\pm$6.42    &  17.88$\pm$7.68     &   17.86$\pm$7.64   &   17.69$\pm$7.14  &  17.76$\pm$6.23   &   17.68$\pm$5.94    &  17.73$\pm$7.04      \\ 

        SCAFFOLD    &   17.22$\pm$2.85   &   17.44$\pm$2.17    &  17.39$\pm$4.02     &   17.20$\pm$3.58   &   16.87$\pm$2.75  &  17.13$\pm$2.79   &   17.00$\pm$2.41    &  17.33$\pm$2.59      \\ 
      
        CE     &   0.15$\pm$0.12   &   0.14$\pm$0.11    &  0.14$\pm$0.07     &   \textbf{0.14$\pm$0.07}   &   0.14$\pm$0.06  &  \textbf{0.14$\pm$0.06}   &   0.12$\pm$0.05    &  \textbf{0.12$\pm$0.05}      \\ \hline
        
        \methodname{}  &   \textbf{0.14$\pm$0.07}   &   \textbf{0.13$\pm$0.06}    &  \textbf{0.13$\pm$0.06}   &   0.14$\pm$0.09    &   \textbf{0.13$\pm$0.07}   &   \textbf{0.14$\pm$0.06}    &  \textbf{0.11$\pm$0.04}    &   0.13$\pm$0.07       \\ \hline

		\end{tabular}
	\end{threeparttable}
	\label{table-exp-synthetic}
\end{table*}

\section{Experimental Evaluation}

We conduct experiments on synthetic data and the CIFAR-10 dataset. To investigate the practicality of \methodname, we also adopt the electronic health record (EHR) dataset eICU \cite{pollard2018eicu} to illustrate the collaboration relationships of FL-PTs on a real-world network of multiple hospitals. 

\subsection{Comparison baselines}
\label{sec.baseline}

Compared with the proposed approach in the last section, we now give a more intuitive procedure to address the competing relationships among FL-PTs. This procedure makes the previous FL approaches (e.g., FedAvg) applicable to the scenario of this paper. 
At a high level, we will find a partition of all FL-PTs into several disjoint groups such that the FL-PTs in each group are independent of each other, without conflict of interest. Then, baselines can be generated by directly applying the previous FL approaches to each group of FL-PTs. 
Specifically, the competing graph $\mathcal{G}_{c}$ describes the competing relationship among FL-PTs. Let $\mathcal{G}_{c}^{-}$ denote the complement of $\mathcal{G}_{c}$: the nonexistence of an edge between $v_{i}$ and $v_{j}$ in $\mathcal{G}_{c}$ leads to the existence of an edge $(v_{i}, v_{j})$ in $\mathcal{G}_{c}^{-}$, and vice versa. Each edge in the graph $\mathcal{G}_{c}^{-}$ indicates that the two FL-PTs connected by this edge are independent. A clique is a subset of nodes of $\mathcal{G}_{c}^{-}$ such that every two nodes in the clique are adjacent, that is, a clique is a subgraph that is complete. A clique cover of $\mathcal{G}_{c}^{-}$ is a partition of all nodes into cliques within which every two nodes in the clique are adjacent and independent of each other \cite{tomita2006worst}. A minimum clique cover is a clique cover that uses as few cliques as possible. %

The FL-PTs in each clique are grouped together to take FL training, without involving the FL-PTs from other cliques. We apply four typical FL approaches directly to the nodes of each clique for FL training: \textbf{FedAvg}, \textbf{CE}, \textbf{FedProx} \cite{MLSYS2020_1f5fe839} and \textbf{SCAFFOLD} \cite{pmlr-v119-karimireddy20a}, which generates four baselines. The collaboration equilibrium (CE) approach is proposed in \cite{Cui22a} where each coalition is defined as a strongly connected component of the benefit graph; its effectiveness has well been validated against several other approaches. FedProx and SCAFFOLD represent two typical approaches that make the aggregated model at the CS close to the global optima and are two benchmarks in \cite{li2022federated} for showing the FL performance under Non-IID data settings. 
The fifth baseline is \textbf{Local} where each FL-PT takes local ML training without collaboration.


\textbf{General experimental setting.} Like \cite{Cui22a}, the hypernetwork technique in \cite{navon2021learning} is used to compute the benefit graph $\mathcal{G}_{b}$ and a hypernetwork is constructed by a multilayer perceptron (MLP). 
When it comes to a specific dataset, all approaches have the same network structure for each FL-PT to execute the learning tasks.


\subsection{Synthetic experiments}

We show the experimental results on synthetic data with fixed competing graphs. Specifically, let us consider 8 FL-PTs $\{v_{1}, v_{2}, \cdots, v_{8}\}$. 
The synthetic features are generated by $x\sim \mathcal{U}[-1.0, 1.0]$. Given the FL-PT $v_i$, the grand truth weights $u_{i,l}=v_{l}+r_{i,l}$ are sampled as $v\sim$ $\mathcal{U}[0.0, 1.0]$ and $r_{i,l}\sim$ $\mathcal{N}(0.0, \rho^{2})$ where $l\in \{1, 2, 3\}$; the noise $\epsilon \sim \mathcal{N}(0.0,$ $0.1^{2})$ is added to each label.

\textbf{Weakly Non-IID setting.} $\rho^{2}$ measures the data distribution discrepancy among FL-PTs. We set $\rho = 0.01$, which means that the generated data are weakly non-iid in terms of sample features and labels. The same type of polynomial regression tasks is learned by all FL-PTs and the synthetic labels are defined as: $y=\sum_{l=1}^{3}{u_{i,l}^{T}x^{l}} +\epsilon$. The network used for predicting the label at each FL-PT is an MLP with one hidden layer. FL-PTs $v_{1}$, $v_{2}$, $v_{5}$ and $v_{6}$ have 2000 samples, while the other FL-PTs have 100 samples. Thus, there exists quantity skew, i.e., a significant difference in the sample quantities of FL-PTs. Two large FL-PTs $v_{1}$ and $v_{2}$ are independent and compete with the other two large FL-PTs $v_{5}$ and $v_{6}$ that are independent. Each small FL-PT competes one large FL-PT: $(v_{1}, v_{7})$, $(v_{2}, v_{8})$, $(v_{3}, v_{5})$, and $(v_{4}, v_{6})$ are edges in the competing graph $\mathcal{G}_{c}$. Such $\mathcal{G}_{c}$ leads to a unique clique cover. Under this setting, the minimum clique cover of $\mathcal{G}_{c}^{-}$ is $\{v_{i}\}_{i=1}^{4}$ and $\{v_{i}\}_{i=5}^{8}$, and small FL-PTs benefit large FL-PTs little. The experimental results (measured by mean squared error (MSE)) are given in Table \ref{table-exp-synthetic}. On average, CE has the worst performance since small FL-PTs $v_{3}$, $v_{4}$, $v_{7}$ and $v_{8}$ cannot benefit from large FL-PTs. Particularly, \methodname{} has the best performance compared with the five baselines.

\begin{table*}[t]
	\centering
		\caption{Experiments with eICU under a fixed competing graph}
	\label{table-exp-hospitial}
	\begin{threeparttable}[t]
 \small
		\begin{tabular}{|C{2.66cm}| C{1.02cm} | C{1.02cm} | C{1.02cm} | C{1.02cm} | C{1.02cm} |C{1.02cm}| C{1.02cm} | C{1.02cm} | C{1.02cm} | C{1.02cm} |}
			\hline
   \multirow{2}{*}{} &  \multicolumn{10}{c|}{AUC} \\ \cline{2-11}    
     &  $v_{1}$   &   $v_{2}$   &   $v_{3}$   &  $v_{4}$   &   $v_{5}$    &   $v_{6}$    &   $v_{7}$   &  $v_{8}$   &  $v_{9}$   &  $v_{10}$    \\ \hline

        Local     &   76.12   &   69.46    &  68.94     &   68.04   &   76.46  &  40.00   &   69.30    &  60.53    & 56.94    &  49.12      \\ 

        FedAvg    &   75.26   &   72.09    &  68.87     &   74.13   &   83.72  &  41.67   &   79.37    &  54.41    & 66.67    &  38.10      \\ 


      
        CE     &    \textbf{83.53}   &   75.64    &    \textbf{74.38}     &   74.46   &    80.89  &  82.61   &   71.43    &  66.67    &  66.67    &  80.00     \\ \hline

        \methodname{}  &  81.50   &   \textbf{78.23}    &  69.18   &    \textbf{83.52}    &   \textbf{85.91}   &    \textbf{89.58}    &   \textbf{80.70}    &   \textbf{68.89}    &  \textbf{90.48}    &  \textbf{95.24}       \\ \hline

		\end{tabular}
	\end{threeparttable}
\end{table*}

\begin{table}[t]
\captionsetup{
  justification = centering
}
	\centering
		\caption{Experiments with CIFAR-10 under randomly generated competing graphs}
	\begin{threeparttable}[b]
 \small
		\begin{tabular}{|C{3.0cm}| C{2.4cm} |}
			\hline
                     &      MTA     \\ \hline
          Local     &         $86.46\pm 4.12$  \\ 
          FedAvg    &         $52.99\pm 4.38$     \\ 
          FedProx    &        $51.13\pm 7.10$     \\ 
          SCAFFOLD    &       $51.20\pm 7.09$     \\ 
          CE        &         $87.80\pm 7.18$    \\ \hline
          \methodname{}   &    \textbf{91.33} $\pm$ \textbf{4.14}     \\ \hline		
          \end{tabular}
	\end{threeparttable}
	\label{table-exp-cifar10}
\end{table}

\textbf{Strongly Non-IID setting.} This setting is the same as the setting above expect three aspects. Firstly, each FL-PT has 2000 samples and there is no quantity skew. Secondly, we generate conflicting learning tasks by flipping over the labels of some FL-PTs: $y = -\sum_{l=1}^{3}{u_{i,l}^{T}x^{l}}+\epsilon$ for $i\in \{5,6,7,8\}$, which leads to strongly Non-IID among the eight FL-PTs in terms of the labels. Thirdly, we test on a different competing graph where there are two independent groups of FL-PTs $\{v_{i}\}_{i=1}^{4}$ and $\{v_{i}\}_{i=5}^{8}$: for $i\in \{1, 5\}$, the FL-PTs $v_{i}$ and $v_{i+1}$ are independent of each other and compete with $v_{i+2}$ and $v_{i+3}$ that are also independent of each other. Under this setting, all FL-PTs in the same group can benefit each other; the minimum clique cover of $\mathcal{G}_{c}^{-}$ is $\{v_{1},$ $v_{2}, v_{5}, v_{6}\}$ and $\{v_{3}, v_{4}, v_{7}, v_{8}\}$. The experimental results are given in Table \ref{table-exp-synthetic}. FedAvg, FedProx, and SCAFFOLD perform the worst since training a global model cannot simultaneously satisfy the FL-PTs in the same clique with conflicting learning tasks. It is observed that \methodname{} has the best performance compared with the five baselines.

\subsection{Benchmark experiments}
\label{sec.exp-cifar10}

We conduct experiments on CIFAR-10 with competing graphs that are generated randomly. CIFAR-10 is an image classification dataset and has 10 classes, each with 6000 images. We follow the setting in \cite{Cui22a} for CIFAR-10 to construct Non-IID data and network structures, and to measure performance. There are 10 FL-PTs, and each FL-PT randomly obtains 2 of the 10 classes to simulate the Non-IID setting. 
The model performance is measured by the mean test accuracy (MTA). To simulate competition, we set the probability of two FL-PTs competing against each other to 0.2, thus generating a random competing graph $\mathcal{G}_{c}$, which constrains the collaboration between some FL-PTs. Table \ref{table-exp-cifar10} shows the experimental results. It is observed that \methodname{} has the best performance. FedAvg, FedProx, and SCAFFOLD perform worst since training a global model cannot simultaneously satisfy the FL-PTs in the same clique with data heterogeneity. \methodname{} performs better than CE by 3.53\%.  

\subsection{Hospital collaboration example}

eICU is a dataset collecting EHRs from many hospitals across the United States admitted to the intensive care unit (ICU). The task is to predict mortality during hospitalization. We use this dataset to illustrate a benefit graph $\mathcal{G}_{b}$ and a data usage graph $\mathcal{G}_{u}$ in the real world. The setting here is the same as the setting in \cite{Cui22a} for eICU, including the data pre-processing procedure, the way of choosing hospitals, the network structures, and the performance metric. There are 10 hospitals, among which the first 5 hospitals $\{v_{i}\}_{i=1}^{5}$ are large with about 1000 patients per hospital and the others are small with about 100 patients per hospital. Label imbalance occurs since more than 90\% samples have negative labels; thus, AUC is used to measure the utility of each FL-PT. 
The generated benefit graph $\mathcal{G}_{b}$ is illustrated in Figure \ref{Fig-Hospitals}(a).

\begin{figure}[t]
\begin{center}
\begin{minipage}[b]{0.245\textwidth}
\begin{center}
\subfigure[$\mathcal{G}_{b}$.]{%
  \includegraphics[width=0.925\textwidth]{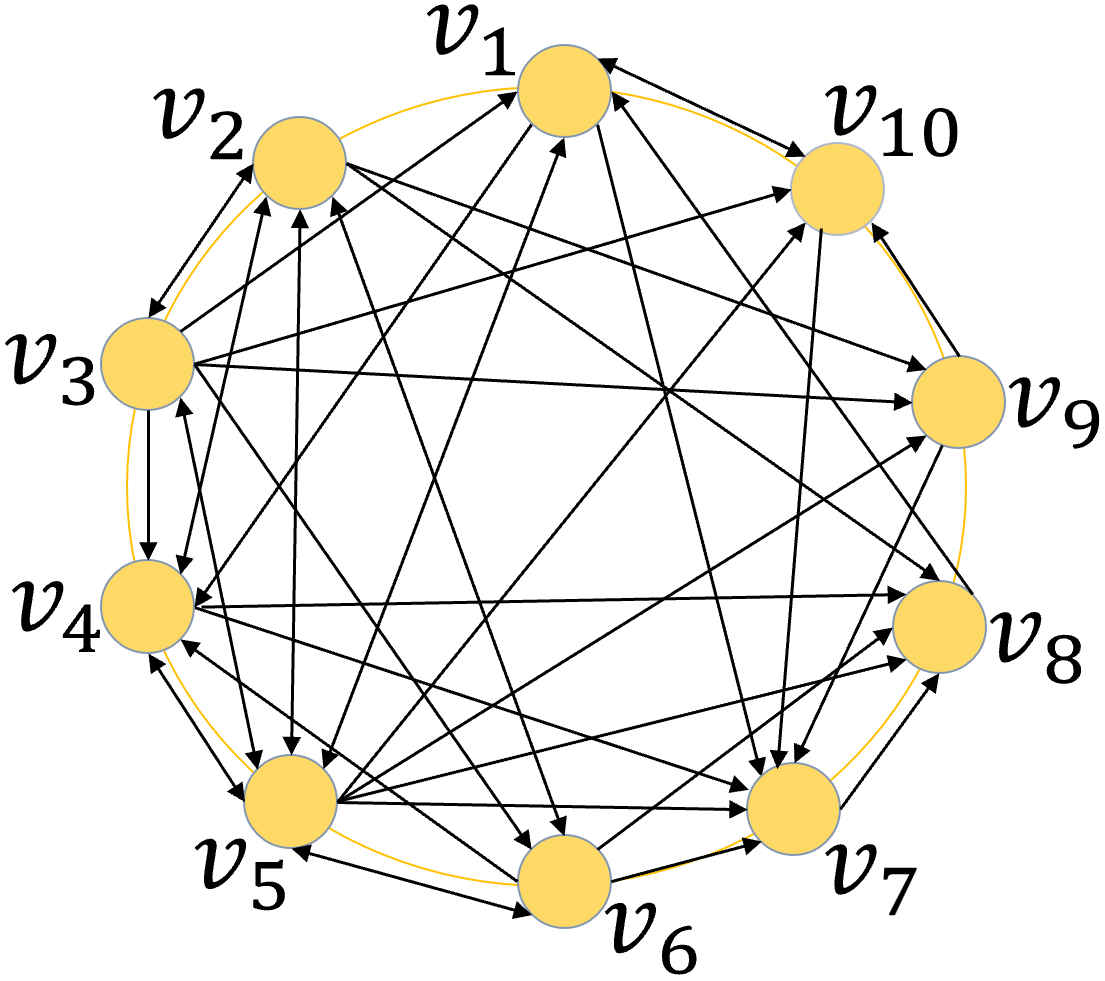}%
}%
\end{center}
\end{minipage}%
\begin{minipage}[b]{0.245\textwidth}
\begin{center}
\subfigure[$\mathcal{G}_{b}$.]{%
  \includegraphics[width=0.925\textwidth]{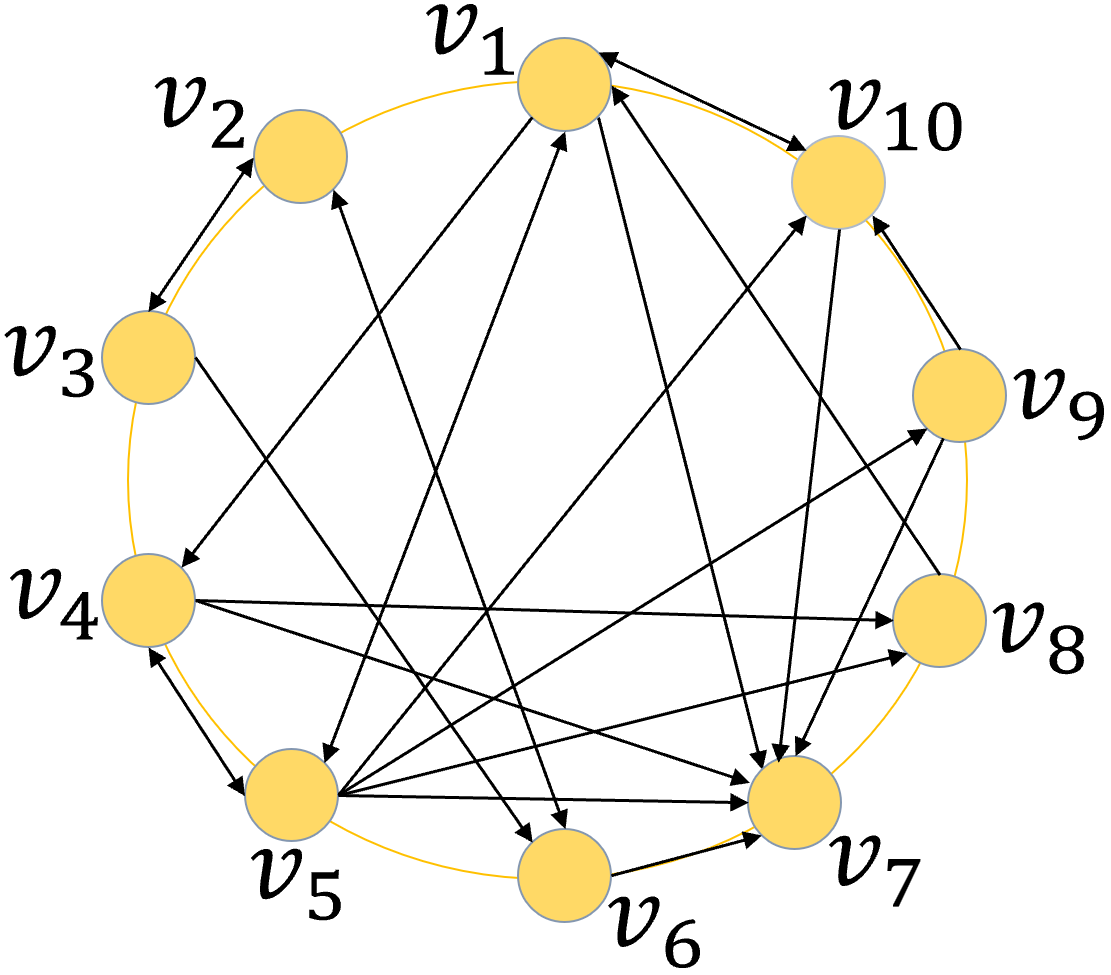}%
}%
\end{center}
\end{minipage}%
\end{center}
\caption{\textbf{Illustration of hospital collaboration.}}
\label{Fig-Hospitals}
\end{figure}

Let us consider the case where more than one large hospital may be located in the same city while small hospitals are dispersed in rural areas with lower population densities; competition mainly occurs among large hospitals. We assume that $v_{2}$ competes with $v_{5}$, while $v_{3}$ competes with $v_{4}$ and $v_{5}$, respectively. For the baselines except the local approach, the way of generating the clique cover is independent of $\mathcal{G}_{b}$ where FL-PTs in each clique collaborate together; the generated clique cover is $\{v_{4}, v_{5}\}$ and $\{v_{i}\}_{i=1}^{3}\cup\{v_{i}\}_{i=6}^{10}$. For \methodname, the generated data usage graph $\mathcal{G}_{u}$ is illustrated in Figure \ref{Fig-Hospitals}(b), which fully utilizes the information on $\mathcal{G}_{b}$ by Algorithm \ref{Greedy}. Compared with the baselines, it is observed from Figure \ref{Fig-Hospitals}(b) that the local model update information of $v_{4}$ and $v_{5}$ can also be utilized by other FL-PTs $\{v_{1}, v_{7}, v_{8}, v_{9}, v_{10}\}$ while $v_{4}$ and $v_{5}$ can similarly benefit from $v_{1}$ in the FL training process. This is an advantage of \methodname{} and is reflected in the experimental results, which are given in Table \ref{table-exp-hospitial}. Overall, \methodname{} achieves the best performance.

\section{Conclusions}

We consider in this paper an open research problem in which a subset of FL-PTs in the FL ecosystem engage in competition. 
We extend a principle from balance theory that ``the friend of my enemy is my enemy” to guarantee that no conflict of interest occurs among FL-PTs. The resulting FL ecosystem thus exhibits a
high level of scalability since FL-PTs that even compete can join smoothly. We formulate the problem and show that it is mathematically solvable in polynomial time. Thus, an efficient algorithm is proposed to determine the collaboration relationships of FL-PTs. 
The framework of this paper is also general since it considers both competition and data heterogeneity, which is another important aspect in FL. Extensive experiments demonstrate the effectiveness of the proposed framework. 


\section{Acknowledgments}
This research was supported in part by the National Key R\&D Program of China (No. 2022YFB2902900). This research/project is also supported, in part, by the National Research Foundation Singapore and DSO National Laboratories under the AI Singapore Programme (AISG Award No: AISG2-RP-2020-019); the RIE 2020 Advanced Manufacturing and Engineering (AME) Programmatic Fund (No. A20G8b0102), Singapore; and the Center for Frontier AI Research (CFAR), Agency for Science, Technology and Research (A$^{\ast}$STAR), Singapore. The work of Hao Cheng and Chongjun Wang was supported by the National Natural Science Foundation of China (Grant No. 62192783, 62376117). The work of Shanli Tan was done when he was a research intern with Xiaohu Wu at the National Engineering Research Center of Mobile Network Technologies, Beijing University of Posts and Telecommunications, China.  



\bibliography{aaai22}

\end{document}